\UseRawInputEncoding
\RequirePackage{fix-cm}
\documentclass[twocolumn]{svjour3}          
\smartqed  

\usepackage{cite}
\usepackage[pdftex]{graphicx}
\usepackage{ragged2e}
\usepackage[tight,footnotesize]{subfigure}
\usepackage{rotating}
\usepackage{graphicx}
\usepackage{amsmath,amssymb} 
\usepackage{array}
\usepackage{multirow}
\usepackage{colortbl}
\usepackage{amsfonts}
\usepackage{pifont}
\usepackage{xspace}
\usepackage{etoolbox}
\usepackage{overpic}
\usepackage{color}
\usepackage{microtype}

\usepackage{enumerate}
\usepackage{multirow}

\newcommand{\orcid}[1]{\href{https://orcid.org/#1}{\includegraphics[width=10pt]{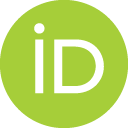}}}

\def\etal{{\em et al}}

\def\bA{\textbf{A}}
\def\ibA{\textit{\textbf{A}}}
\def\ibB{\textit{\textbf{B}}}
\def\ibF{\textit{\textbf{F}}}
\def\ibT{\textit{\textbf{T}}}
\def\ibU{\textit{\textbf{U}}}
\def\ibV{\textit{\textbf{V}}}

\def\hatA{\hat{\textit{\textbf{A}}}}

\def\iba{\textit{\textbf{a}}}
\def\ibf{\textit{\textbf{f}}}
\def\ibu{\textit{\textbf{u}}}
\def\ibup{\textit{\textbf{u}}'}

\usepackage{hyperref}
\hypersetup{breaklinks=true,citecolor=blue, colorlinks}

\graphicspath{{./Imgs/}}
\DeclareGraphicsExtensions{.pdf,.jpg,.png}

\usepackage{silence}
\journalname{Research Article}

\begin{document}

\title{A revisit of the normalized eight-point algorithm and a self-supervised deep solution}

\titlerunning{A revisit of the normalized eight-point algorithm and a self-supervised deep solution}        

\author{Bin Fan$^{1}$ \orcid{0000-0002-8028-0166} \and
  Yuchao Dai$^1$ \orcid{0000-0002-4432-7406} \and 
  Yongduek Seo$^2$ \orcid{0000-0002-0570-2197} \and 
  Mingyi He$^1$ \orcid{0000-0003-2051-6955}
}

\authorrunning{B. Fan \etal} 

\institute{
$^1$ School of Electronics and Information, Northwestern Polytechnical University and Shaanxi Key Laboratory of Information Acquisition and Processing, Xi'an, China (Email: binfan@mail.nwpu.edu.cn, \{daiyuchao, myhe\}@nwpu.edu.cn). \\
$^2$ Department of Art and Technology, Sogang University, Seoul, Korea (Email: yndk@sogang.ac.kr). \\
Corresponding author: Yuchao Dai.
}

\date{Received: date / Accepted: date}

\maketitle

\begin{abstract}
The normalized eight-point algorithm has been widely viewed as the cornerstone in two-view geometry computation, where the seminal Hartley's normalization has greatly improved the performance of the direct linear transformation algorithm. 
A natural question is, whether there exists and how to find other normalization methods that may further improve the performance as per each input sample.
In this paper, we provide a novel perspective and propose two contributions to this fundamental problem: 1) we revisit the normalized eight-point algorithm and make a theoretical contribution by presenting the existence of different and better normalization algorithms; 
2) we introduce a deep convolutional neural network with a self-supervised learning strategy for normalization. Given eight pairs of correspondences, our network directly predicts the normalization matrices, thus learning to normalize each input sample.
Our learning-based normalization module can be integrated with both traditional (e.g., RANSAC) and deep learning frameworks (affording good interpretability) with minimal effort.
Extensive experiments on both synthetic and real images demonstrate the effectiveness of our proposed approach.


\end{abstract}

\section{Introduction}
\label{sec:intr}
Geometric computation has long been one of the major issues in computer vision. 
In particular, two-view geometry computation is a central building block for three-dimensional (3D) modeling and camera motion estimation. For example, self-driving is implemented through the technology of simultaneous localization and mapping (SLAM) and structure from motion (SfM).
Among many important core algorithms, the eight-point algorithm \cite{Longuet_Reconstructing_Nature_1981} computes the fundamental matrix from a set of eight or more point correspondences between two views, which has the advantage of the simplicity of implementation.
However, it was extremely susceptible to image noise and hence was of very limited practical use until Hartley devised a normalized eight-point algorithm in his seminal work \cite{Hartley_Normalization_TPAMI_1997}, which shows that by preceding the algorithm with a data normalization (translation and scaling) of the coordinates of the correspondences, the results obtained are comparable to those of the best iterative algorithms.
%
As a consequence, with its simple strategy of translation and scaling, the isotropic normalization, now termed as {Hartley's normalization}, has gradually become an indispensable component of many geometric computations not only for fundamental matrix estimation \cite{dai2016rolling} but also for homography \cite{zhao2021homography}, ellipse fitting \cite{szpak2015guaranteed}, bundle adjustment \cite{zhang2014structure}, etc. 

One particular aspect of Hartley's normalization in regard to the direct linear transformation (DLT) formulation of the fundamental matrix computation is that it allows the DLT solution to possess a better condition number. Therefore, when the solution matrix is enforced to have rank 2, a much more stable estimate of the fundamental matrix is obtained; this is important because it is the starting point of all the structure and motion computations such as guided correspondence search, camera and structure optimization, and 3D reconstruction for more than two views. 
Consequently, enforcing the rank-2 constraint as much as possible at the DLT stage becomes an interesting topic of study. For example, M{\" u}hlich and Mester \cite{Muhlich_Subspace_SCIA_2001} performed a statistical analysis to obtain an optimal data normalization for DLT fundamental matrix computation and showed that Hartley's normalization can be expected to work well even though it is not identical to the optimal transform. 
Mair et al. \cite{Elamr_ErrorPropagation_IROS_2013} performed further error analysis to obtain a better performance than Hartley's eight-point algorithm.  
The work of da Silveira and Jung \cite{daSilveira_Perturbation_CVPR_2019} presented a perturbation analysis of the eight-point algorithm for a wide field of view cameras.
In contrast to these works based on statistical analysis, this paper tries to determine the mechanism of data normalization through deep learning without specific statistical modeling.
Considering that the fundamental matrix estimation is strongly affected by the error distribution of the feature matching algorithm, we argue that a data normalization scheme can be exploited to achieve DLT solutions of improved rank-2 condition by learning the error distribution from the data themselves; this approach coincides with the views of Refs. \cite{Muhlich_TLS_ECCV_1998,Muhlich_Subspace_SCIA_2001,daSilveira_Perturbation_CVPR_2019}. 
In particular, as displayed in Fig.~\ref{fig:fig1}, we propose to learn a {data-driven normalization} scheme under the standard configuration of eight correspondences.

Currently, the success of deep learning in high-level vision tasks has been gradually extended to multi-view geometry problems such as homography \cite{Detone_DeepHomography_arxiv_2016}, fundamental matrix \cite{Ranftl_DeepFundamental_ECCV_2018}, bundle adjustment \cite{Tang_BANet_ICLR_2018}, plane sweeping \cite{Sunghoon_DPSNet_ICLR_2019,fan2021rs}, and rolling-shutter modeling \cite{fan2023rolling,fan2022rolling}.
However, this success has not been extended to the normalized eight-point algorithm and a different or better normalization scheme has so far not been presented nor replaced by the deep learning pipeline. 
This is mainly due to the following obstacles: {1}) gradient descent cannot be trivially applied as mentioned in Ref. \cite{Ranftl_DeepFundamental_ECCV_2018}; {2}) the network must be invariant to the permutation of the correspondences, {i.e.}, different orderings of the input data should produce the same normalization;
and {3}) a large amount of labeled input and output data should be used for supervised learning (in this case, the input is eight-point correspondences and the output is the optimal data normalization).
In this paper, we overcome these problems by back-propagating through a singular value decomposition (SVD) layer and 
using a self-supervised learning mechanism in the permutation-invariant network architecture; this also solves the issue of large training data requirements.
Our approach not only produces an interpretable pipeline of fundamental matrix estimation but can also be easily embedded in other robust frameworks such as 
the differentiable random sample consensus (RANSAC) \cite{Brachmann_DSAC_CVPR_2017}.
%
%
Through experiments, our learning-based normalization demonstrated superior performance to Hartley's normalization and a good generalization ability across different datasets.
%
Our main contributions can be summarized as follows.
\begin{itemize}
	\item[1)] We propose a self-supervised learning-based deep solution for normalizing DLT fundamental matrix estimation under the standard configuration of eight point correspondences.
	\vspace{2mm}
	\item[2)] We make a theoretical contribution by demonstrating the existence of different and better normalization algorithms beyond Hartley's normalization.
        \vspace{2mm}
	\item[3)] Extensive experiments on both synthetic and real images demonstrate the effectiveness and good generalizability of our proposed approach.
\end{itemize}

\begin{figure*}[!t]
	\centering
	\includegraphics[width=0.867\textwidth]{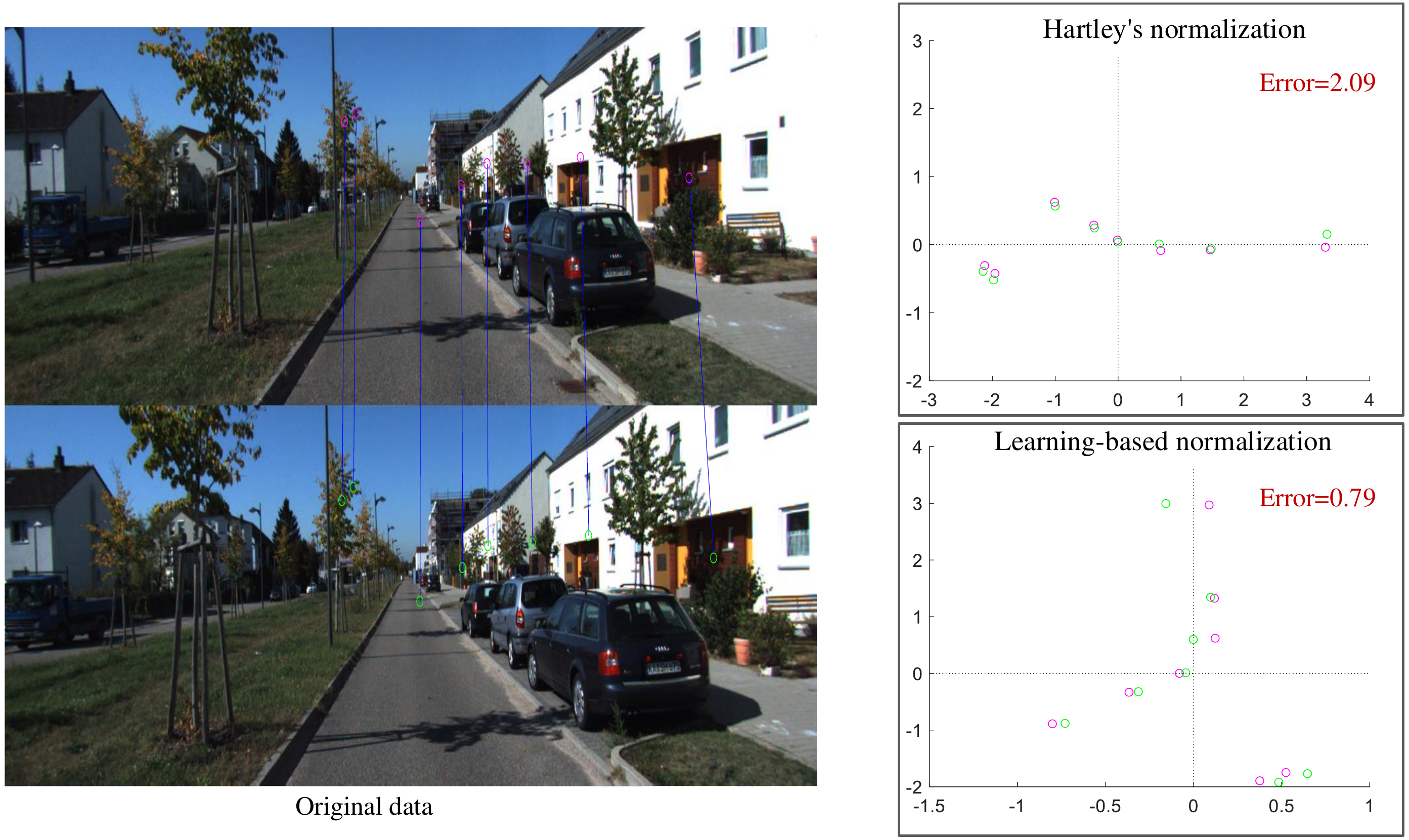}
	\caption{Distributions of normalized image coordinates by using Hartley’s normalization algorithm (upper right) and our learning-based normalization approach (bottom right), respectively. Eight pairs of point correspondences are obtained from the two street images on the left.
 Note that in the right figure, the coordinate axes represent the normalized image coordinates in the horizontal and vertical directions, and the ``error'' refers to the symmetry epipolar distance, which can better characterize the estimation accuracy of the two-view geometry.
		Our approach learns a robust normalization scheme adapted to the input data, obtains a better distribution spread of the normalized point coordinates, and eventually leads to improved performance in the computation of the fundamental matrix. 
		\label{fig:fig1}}
  \vspace{-2mm}
\end{figure*}

\section{Related work}
In this section, we briefly review related work in traditional two-view geometry computation and deep learning-based multi-view geometry learning.

\subsection{Two-view geometry estimation}
The normalized eight-point algorithm \cite{Hartley_Normalization_TPAMI_1997} significantly improves the numerical accuracy of the fundamental matrix and extends the scope of applications due to the improved condition number of the hand-designed normalization scheme. Since this seminal work, there have been various follow-up studies on the uncertainty in fundamental matrix estimation and the relationships between the epipolar constraint and corresponding errors. 
Csurka {et al.} \cite{Csurka_Uncertainty_CVIU_1997} proposed a method to simultaneously estimate the fundamental matrix and its uncertainty. M{\" u}hlich and Mester \cite{Muhlich_TLS_ECCV_1998} concluded that the normalization strategy can ensure that the two-view non-iterative motion estimation algorithm maintains unbiasedness and consistency.
They further introduced a normalization transformation scheme based on the bound of epipolar constraint errors obtained by assuming known feature matching covariance, which was also used to extend the existing first-order error propagation analysis of the eight-point algorithm in Ref. \cite{Elamr_ErrorPropagation_IROS_2013}.
However, this approach was still not optimal because the error distribution of the input data was not considered \cite{Muhlich_Subspace_SCIA_2001}. The closed-form computation of the uncertainty of the fundamental matrix was presented in Ref. \cite{Frederic_Uncertainty_BMVC_2008} to recover correspondences via the uncertain equilibrium of motion estimation. Chojnacki and Brooks \cite{Chojnacki_Revisiting8pt_TPAMI_2003} revisited the normalized eight-point algorithm and presented a statistical model of data distribution by merging the statistical approach of Ref. \cite{Muhlich_TLS_ECCV_1998}, which was further extended in Ref. \cite{Chojnacki_Consistency_JMIV_2007} by introducing a structured model for the data distribution. 
In addition, da Silveira and Jung \cite{daSilveira_Perturbation_CVPR_2019} performed perturbation analysis for the fundamental matrix estimation without considering any kind of matching error distribution.

\subsection{Deep learning-based geometry estimation}
Recently, the success of deep learning in high-level vision tasks has been gradually extended to various multi-view geometry estimation problems. DeTone {et al.} \cite{Detone_DeepHomography_arxiv_2016} employed a deep convolutional neural network (CNN) to regress a homography from a pair of input images in an end-to-end manner. A follow-up study \cite{Nguyen_DeepHomographyUnsupervised_Robot&Automation_2018} developed the unsupervised variant by replacing direct supervision with image-based loss. This pipeline has been extended to fundamental matrix estimation, where a fundamental matrix is directly regressed from a pair of stereo images without correspondences \cite{Omid_DeepFundamental_wo_corresponences_ECCV_2018}.
Ranftl and Koltun \cite{Ranftl_DeepFundamental_ECCV_2018} treated the fundamental matrix estimation problem as a weighted homogeneous least-squares problem, where the matching weights and fundamental matrix are simultaneously estimated by using supervised deep networks. With the availability of camera intrinsics, Yi {et al.}  \cite{Yi_LearningCorrespondences_CVPR_2018} recovered the essential matrix from putative correspondences with little training data and limited supervision, thus finding good correspondences for wide-baseline stereo.
Furthermore, Probst {et al.} \cite{Probst_2019_CVPR} proposed an unsupervised learning framework for consensus maximization, in the context of solving 3D vision problems such as 3D-3D matching \cite{zhang2022learning,zhang2022searching}, and image-to-image matching (homography and fundamental matrix).
DSAC \cite{Brachmann_DSAC_CVPR_2017} is a differentiable counterpart of RANSAC and can also be leveraged as a robust optimization component for other deep learning pipelines.



Different from existing work in deep learning-based multi-view geometry computation, our self-supervised learning strategy removes the need for supervisory signals and thus generalizes well across different datasets. Furthermore, our learning-based normalization module can be integrated with both traditional and deep learning frameworks.


\section{A revisit of the normalized eight-point algorithm}\label{sec:revisit}

We use capital letters, $\ibA$, $\ibB$, etc., to denote matrices. The operation of reshaping a matrix into a vector is denoted by $\mathrm{vec}(\cdot)$, defined as $\mathrm{vec}(\ibA) = [\iba_1^T,...,\iba_N^T]^T$, where $\iba_i$ is the $i$-th column vector of $\ibA$ and $N$ is the number of columns. Its inverse operation is denoted as $\mathrm{mat}(\cdot)$.

Given a pair of correspondences $\textit{\textbf{u}}'_i$ and $ \textit{\textbf{u}}_i$ between two views, the epipolar constraint is expressed as
\begin{align}\label{eq:1}
{\textit{\textbf{u}}'_i{^T}} \textit{\textbf{F}} \textit{\textbf{u}}_i = 0,
\end{align}
where $\ibF = [f_{ij}]$ is a $3 \times 3$ matrix of rank 2, termed as the fundamental matrix.
Collecting $N=8$ point correspondences $\{ (\ibup_i, \ibu_i) | i=1,\ldots,8\}$,
{i.e.}, the standard configuration,
we may rewrite Eq.~\eqref{eq:1} as a linear equation of $\ibf$:
\begin{equation}\label{eq:3}
\centering
{\textit{\textbf{A}}} \textit{\textbf{f}} = 0,
\end{equation}
where $\ibf = \mathrm{vec}(\ibF^T)$ is a nine-dimensional vector composed of stacked columns of $\ibF^T$, and $\ibA = [\iba_1, \ldots, \iba_8]^T$ is the $8\times9$ coefficient matrix with $\iba_i = \mathrm{vec}(\ibu_i {\textit{\textbf{u}}'}_i^{T})$ for $i=1,\ldots,8$. 
This approach provides the DLT formulation for computing $\ibF$, and a solution may be obtained through SVD of $\ibA$.

Despite its simplicity, the computation of the DLT for the eight-point algorithm \cite{Longuet_Reconstructing_Nature_1981} is extremely susceptible to noise in the image coordinate measurements.
%
In the seminal work \cite{Hartley_Normalization_TPAMI_1997}, Hartley showed that the precision of the eight-point algorithm can be greatly improved by proper normalization of the image coordinates; this approach is the classic {normalized eight-point algorithm}.
%
%
Hartley's normalization is designed to compute image translation and scaling such that the average distance of the transformed coordinates from the origin is $\sqrt 2$:
\begin{equation}\label{eq:51}
\textit{\textbf{T}}_{\rm H} = 
\left[ {\begin{array}{*{20}{c}}
	s&{}&{ - so_1}\\
	{}&s&{ - so_2}\\
	{}&{}&1
	\end{array}} \right],
\end{equation}
with $s$, $o_1$ and $o_2$ given by
\begin{equation}
o_j = \frac{1}{N}\sum\limits_{i = 1}^N {\textit{\textbf{u}}_i^{{{(j)}}}}
\,\,\,\,\,{\rm{and}}\,\,\,\,\,
s = \frac{{\sqrt 2 }}{{\frac{1}{N}\sum\limits_{i = 1}^N {\left\| {\textit{\textbf{u}}_i - \textit{\textbf{o}}} \right\|_2} }},
\end{equation}
where the superscript $ {j} $ denotes the $j$-th entry of vector $ \textit{\textbf{u}}_i $.
Given two normalization matrices $\textit{\textbf{T}}'$ and $\textit{\textbf{T}}$, 
Eq.~\eqref{eq:3} is transformed to
\begin{equation}\label{eq:4}
\hat{{\textit{\textbf{A}}}} \hat{\textit{\textbf{f}}} = \mathbf{0},
\end{equation}
where $\hat{{\textit{\textbf{A}}}} = [\hat{{\textit{\textbf{a}}}}_1,...,\hat{{\textit{\textbf{a}}}}_8  ]^T $ is the transformed coefficient matrix with $ \hat{{\textit{\textbf{a}}}}_i = {\rm{vec}}(\hat{\textit{\textbf{u}}}_i\hat{\textit{\textbf{u}}}'_i{^T}) 
= {\rm{vec}}( \textit{\textbf{T}}\textit{\textbf{u}}_i  \textit{\textbf{u}}'_i{^T}\textit{\textbf{T}}'^T )$.
In summary, the {normalized eight-point algorithm} mainly includes the following three steps.
\begin{enumerate}[1)]
	\item Normalization: Transform the input image coordinates according to $\hat{\textit{\textbf{u}}}'_i = \textit{\textbf{T}}'\textit{\textbf{u}}'_i$ and $\hat{\textit{\textbf{u}}}_i = \textit{\textbf{T}} \textit{\textbf{u}}_i$.
	\item Compute the corresponding fundamental matrix $\hat{{\textit{\textbf{F}}}}'$ to normalize data by
	\begin{enumerate}[a)]
		\item Direct linear transform: Determine $\hat{{\textit{\textbf{F}}}} = {\rm{mat}}(\hat{\textit{\textbf{f}}}) $ from the right singular vector $\hat{\textit{\textbf{f}}}$ corresponding to the smallest singular value of $ \hat{\ibA} $ defined in Eq.~\eqref{eq:4}.
		\item Singularity constraint enforcement: Replace $ \hat{{\textit{\textbf{F}}}} $ by $\hat{{\textit{\textbf{F}}}}'=\hat{{\textit{\textbf{U}}}}{\rm{diag}}(r_1,r_2,0)\hat{{\textit{\textbf{V}}}}^T$, where $\hat{{\textit{\textbf{F}}}} = \hat{{\textit{\textbf{U}}}} \hat{{\textit{\textbf{D}}}} \hat{{\textit{\textbf{V}}}}^T$ with $ \hat{\textit{\textbf{D}}} $ is a diagonal matrix $ \hat{\textit{\textbf{D}}} = {\rm{diag}}(r_1,r_2,r_3)$ satisfying $ r_1 \ge r_2 \ge r_3 $.
	\end{enumerate}
	\item Denormalization: Set $ {{\textit{\textbf{F}}}} = {{\textit{\textbf{T}}}}'^T \hat{{\textit{\textbf{F}}}}' {{\textit{\textbf{T}}}} $.
\end{enumerate}

The condition number of $\textit{\textbf{A}}$ is defined as ${\kappa }(\textit{\textbf{A}}) = \left\| \ibA \right\|_2 \left\| \ibA^{+} \right\|_2$, where $\ibA^+$ is the pseudo-inverse of $\ibA$.
Its equivalent condition number may be defined as the ratio of the greatest to the second smallest singular values, ${\kappa }(\textit{\bA}) = \sqrt {d_1/d_8}$, for $\textit{\textbf{A}}^T\textit{\textbf{A}} = \textit{\textbf{U}}\mathrm{diag}(d_1,d_2,...,d_8,d_9)\textit{\textbf{U}}^T$. 
%
%
It has been reported in the literature \cite{Hartley_Normalization_TPAMI_1997,Chojnacki_Revisiting8pt_TPAMI_2003,Chojnacki_Consistency_JMIV_2007,daSilveira_Perturbation_CVPR_2019}
that the unsatisfactory performance of the eight-point algorithm is mainly due to the worse numerical conditioning of the coefficient matrix $\ibA$.
In fact, the condition number $k(\ibA)$ is extremely large, leading to two least eigenvalues relatively close to one another, and causing their corresponding eigenvectors to be mixed up and indistinguishable. 
As a result, a negligible perturbation of the matrix entries tends to cause a significant change in the smallest eigenvector, since it may fall anywhere in the proximity to the eigensubspace spanned by the similar eigenvectors associated with those virtual degenerate eigenvalues \cite{Chojnacki_Revisiting8pt_TPAMI_2003}.
It has been found that proper selection of normalization to the input image coordinates results in better numerical conditioning when carrying out linear DLT computation, and that the improved numerical conditioning provides with the smallest eigenvector of $\hat{\ibA}$ far less susceptible to interference \cite{Hartley_Normalization_TPAMI_1997,Chojnacki_Consistency_JMIV_2007}.
From this point, a natural question arises: {\it Can we achieve the ultimate optimal condition number $k(\hat\ibA)=1$?}
Below we figure out that the condition number of the transformed coefficient matrix cannot reach the optimum of 1.
A follow-up question must be: {\it Can we have a better normalization transformation?}
This paper provides a positive answer in the next section. We develop a self-supervised CNN-based technique that learns the convolutional neural network weights based on a geometric loss function. It requires no ground truth labeling but has shown highly improved performance in various experiments.


\begin{proposition}\label{p1}
	There is no pair of normalization matrices $\textit{\textbf{T}}'$ and $\textit{\textbf{T}}$ that results in $k(\hat{\ibA}) = 1$.
\end{proposition}

\begin{proof}
	(Proof by contradiction) For the full row rank matrix $\ibA$, there must be an invertible matrix $\textit{\textbf{P}} = [p_{ij}] $ such that $\ibA=\textit{\textbf{P}}^{-1}\textit{\textbf{Q}}$ holds, where the matrix $\textit{\textbf{Q}} \in \mathbb{R}^{8\times9}$ also has full row rank \cite{Horn_MatrixAnalysis_2012}. Moreover, one can assume that each row of $\textit{\textbf{Q}}$ represents a standard orthonormal basis of the $9$-dimensional subspace, which is easily achieved by matrix decomposition \cite{Horn_MatrixAnalysis_2012}, such as Gram-Schmidt orthogonalization, QR decomposition, and SVD decomposition.
	
	The condition number ${\kappa (\hat{\textit{\textbf{A}}}) =1}$ if and only if $\hat{\ibA}\hat{\ibA}^T = c\textit{\textbf{I}}$, where $c$ is a non-zero positive constant \cite{Horn_MatrixAnalysis_2012,Chen_Cond_ECNU_1986}; this implies that the rows of $\hat{\textit{\textbf{A}}}$ make up eight orthogonal bases 
	of the $9$-dimensional subspace up to a fixed-length scale $\sqrt{c}$.
	Therefore, in order to achieve 
	$\kappa (\hat{\textit{\textbf{A}}}) = 1$, 
	the two invertible transformations $\textit{\textbf{T}}'$ and $\textit{\textbf{T}}$ should make $\hatA = \textit{\textbf{Q}} = \textit{\textbf{P}}\ibA$ hold, {i.e.},
	\begin{equation}\label{eq:5}
	{\rm{mat}}(\hat{\textit{\textbf{a}}}_i^T) 
	= 
	\sum\limits_{j = 1}^8 {{p_{ij}} \cdot {\rm{mat}}({\iba_j^T}}), \quad     i=1,...,8.  
	\end{equation}
	%
	Note that 
	$   
	{\rm{rank}}({\rm{mat}}(\hat{\textit{\textbf{a}}}_i^T)) =
	{\rm{rank}}({\rm{mat}}(\iba_{i}^T)) = 1
	$
	=
	$
	{\rm{rank}}(\hat{\textit{\textbf{u}}_i}  \hat{\textit{\textbf{u}}}'_i{^T}) =
	{\rm{rank}}(\textit{\textbf{u}}_i \textit{\textbf{u}}'_i{^T})
	$ 
	for any $ i\in[1,8]. $
	%
	Except for the trivial configuration in which ${\textit{\textbf{P}}} = \textit{\textbf{I}}$, the rank of the sum on the right-hand-side must exist to be equal to 3 for any $\ibT'$ and $\ibT$ ({e.g.} given in Eq.~\eqref{eq:51}); so Eq.~\eqref{eq:5} cannot be established. That is, there are no normalization matrices $\ibT'$ and $\ibT$ to make ${\kappa (\hat{\textit{\textbf{A}}}) =1}$ tenable.
\end{proof}

\section{Learning-based normalization with self-supervised CNNs}



This section develops a machine learning model that produces $\ibT$ and $\ibT'$, the two data normalization matrices, which result in a better estimation of ${\ibF}$ than Hartley's normalization for eight input correspondences.
As discussed in Section \ref{sec:revisit}, the estimation of the fundamental matrix has two main steps. First, the input image coordinates are normalized by $\ibT$ and $\ibT'$ to construct the data matrix $\hat{\ibA}$, and the solution $\hat{\ibf}$ is obtained. Second, $\hat{\ibF}$ is reconstructed by enforcing the singularity constraint. 
The following are two observations regarding this estimation process:
\begin{enumerate}[1)]
	\item 
	The goal of Hartley's normalization is to achieve a better computation of $\hat{\ibf}$. However, this does not guarantee the singularity condition $\mathrm{det}(\mathrm{mat}(\hat{\ibf})) = 0$, which is why the singularity constraint enforcement (SCE) is necessary.
	\item 
	There are cases where enforcing the singularity ($ \hat{\ibF} = \hat{\ibU}\mathrm{diag}(r_1,r_2,0)\hat{\ibV}^T$) brings about large nonlinear projection errors and leads to an unsatisfactory estimation of $\hat{\ibF}$. This happens especially when $\rho=r_2/r_3$ is not large enough. 
\end{enumerate}
It is evident that the singularity constraint should be considered at the same time as well as the numerical conditioning when the normalization matrices $\ibT$ and $\ibT'$ are prepared, which implies the existence of better normalization schemes.

Our approach adopts a CNN-based model and a self-supervised learning algorithm to train it. The model outputs the parameters of the normalization matrices when eight input correspondences are provided as input. 
Following the conjecture of the affine structure of the normalization matrix proposed in Ref. \cite{Muhlich_TLS_ECCV_1998}, the normalization matrix is designed here to have two more parameters than Hartley's normalization: 
\begin{equation}
\ibT_L =
\begin{bmatrix} \alpha_1 & \\ & \alpha_2 & \\ & & 1 \end{bmatrix}
\begin{bmatrix} \cos\theta & -\sin\theta & \\ \sin\theta & \cos\theta & \\ & & 1
\end{bmatrix}
\begin{bmatrix}
1 & & -o_1 \\ & 1 & -o_2 \\ & & 1
\end{bmatrix},
\end{equation}
which can characterize the data distribution better and enable more general normalization schemes to be implemented by CNNs. Nevertheless, how to robustly determine three normalization parameters (especially $\alpha_1$, $\alpha_2$, and $\theta$) has always been a difficult problem.
Note that, after Hartley's seminal solution \cite{Hartley_Normalization_TPAMI_1997}, there has been no substantial progress in designing hand-crafted normalization strategies.
In contrast, we try to extend Hartley's normalization and develop a deep solution for normalization.
The performance of the CNN model for this parametrization is evaluated and visualized through various experiments in Section \ref{sec:experiments}.
The overall computation pipeline of our framework is illustrated in Fig.~\ref{fig:fig2}.
\begin{figure}[htbp]
	\centering
	\includegraphics[width=0.47867\textwidth]{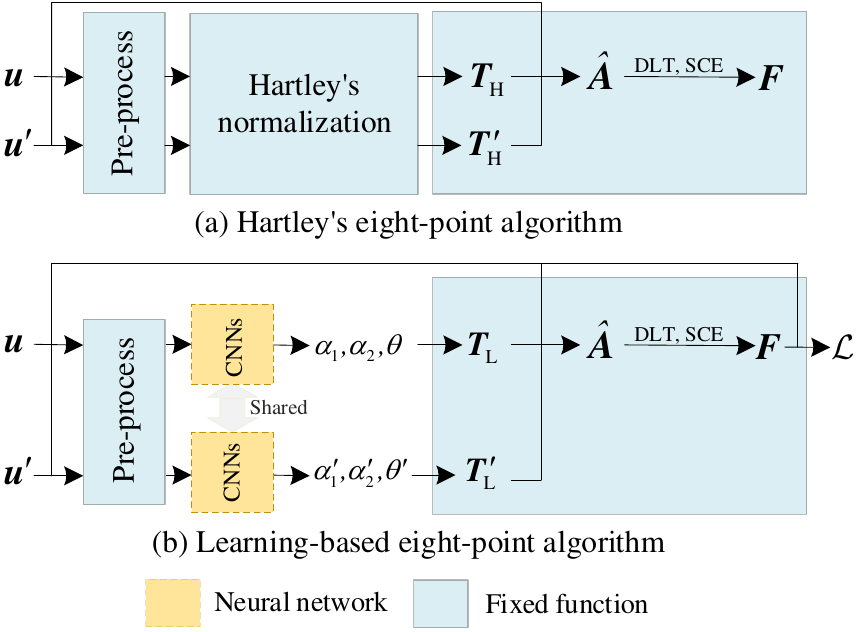}
	\caption{Overall framework comparisons of Harley's eight-point algorithm and our learning-based eight-point algorithm, both of which support eight points as input. Our approach shows an interpretable pipeline to predict the parameters of each normalization matrix ($\alpha_1$, $\alpha_2$, and $\theta$ in particular), which is also beneficial for a more accurate estimation of the intrinsic epipolar geometry. DLT refers to the direct linear transformation and SCE refers to the singularity constraint enforcement.
	\label{fig:fig2} }
        \vspace{-4.0mm}
\end{figure}

\subsection{Self-supervised learning for normalization}

\begin{figure*}[!t]
	\centering
	\includegraphics[width=0.833\textwidth]{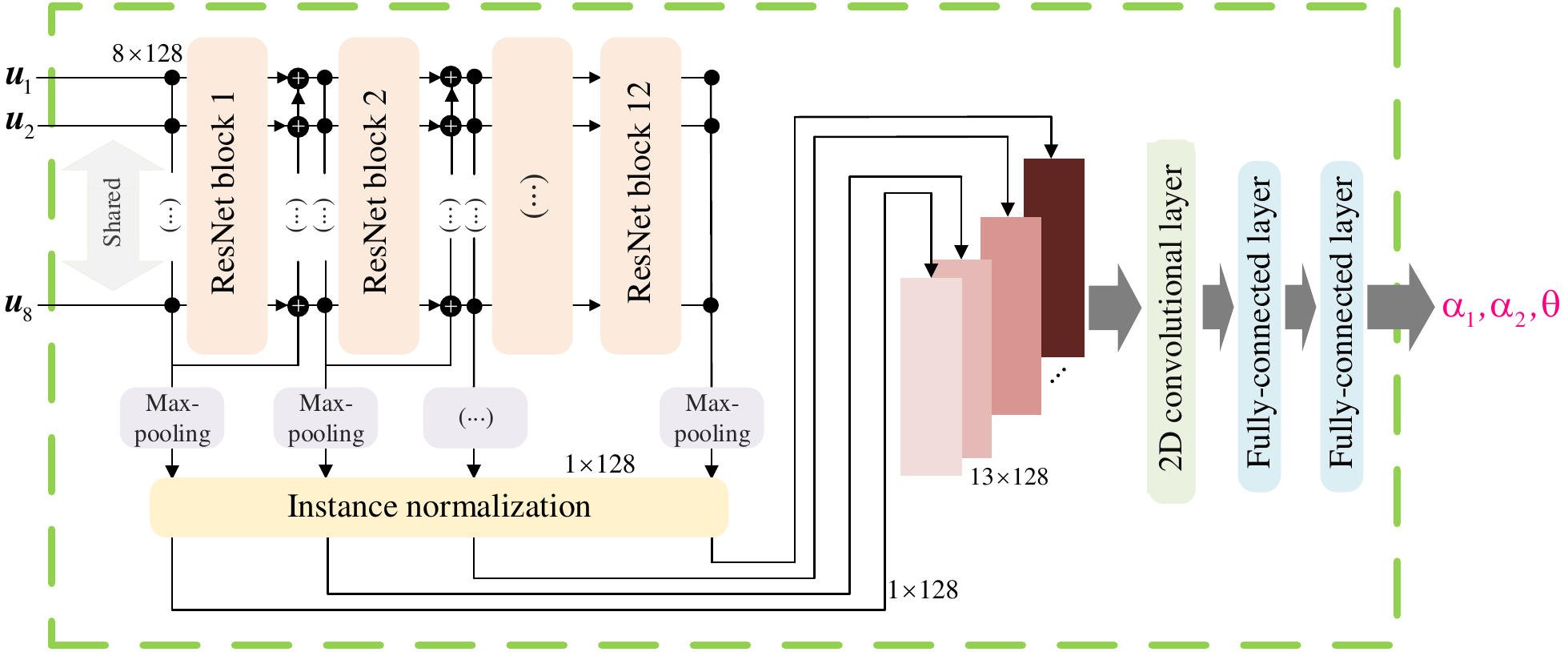}
	\caption{Overview of our network architecture, corresponding to the CNN layer in Fig.~\ref{fig:fig2}. Our approach estimates the parameters of the normalization matrix ($\alpha_1$, $\alpha_2$, and $\theta$ in particular). 2D convolutional layer refers to two-dimensional convolutional layer.
	\label{fig:fig3}}
\end{figure*}

\noindent\textbf{Network architecture.}
%
The overall network architecture is illustrated in Fig.~\ref{fig:fig3}. 
We adopt the structure of the 12 consecutive ResNet blocks as the first stage of the CNN network, which is consistent with the classic two-view geometry estimation networks \cite{Yi_LearningCorrespondences_CVPR_2018,Ranftl_DeepFundamental_ECCV_2018}.
The eight input points $\textit{\textbf{u}}'$ or $\textit{\textbf{u}}$ are first processed by multi-layer perceptrons of 128 neurons sharing weights \cite{Yi_LearningCorrespondences_CVPR_2018} between correspondences. Then, the 128-dimensional features for each correspondence are transmitted as output through 12-layer ResNet blocks \cite{He_ResNet_CVPR_2016,Yi_LearningCorrespondences_CVPR_2018}. The integration of global information is performed by weight-sharing operations between different correspondences, followed by instance normalization \cite{Ulyanov_TextureNetwork_CVPR_2017} after each layer. 
Max-pooling and instance normalization are applied to each layer of the 12-layer ResNet blocks, namely, the input of the first ResNet block and the output of each of the next 12 ResNet blocks, to extract 13 global features of $1\times128$ dimensions, respectively. This process enables the CNN layer to maintain the permutation invariance and fix the size of the global feature maps. Then, 13 feature maps are concatenated and delivered to the two-dimensional (2D) convolutional layer, which consists of eight channels, $3\times3$ square kernels, and unequal strides with four in the column and one in the row. The output of the 2D convolution is then passed through two fully-connected layers each with a dimension of 256, followed by ReLU. Finally, three-parameter estimation corresponding to $\textit{\textbf{u}}'$ or $\textit{\textbf{u}}$ is regressed.
Note that our network supports the input of more than eight correspondences and this flexibility is mainly due to the max-pooling and instance normalization design, which is valuable in practice.

Our network is inspired by 3DRegNet \cite{Pais_3DRegNet_arxiv_2019} but has significant differences in architecture design: we utilize weight sharing for point correspondences, instance normalization module for better performance, and fewer parameters in 2D convolution.
Specifically, compared to the representative two-view geometry estimation methods \cite{Yi_LearningCorrespondences_CVPR_2018,Ranftl_DeepFundamental_ECCV_2018}, our network is invariant to the permutation of the correspondences.


\vspace{0.5mm}
\noindent\textbf{Self-supervised learning.}
In order to train our model through self-supervised learning, the outputs obtained from the CNN model are leveraged to construct the normalization matrices $\ibT$ and $\ibT'$, and are fed into the next module performing 1) the data scaling, 2) DLT to compute $\hat{\ibf}$, and 3) SVD to compute singularity constrained $\hat{\ibF}$.
Finally, the output $\ibF$ is evaluated using the loss function chosen to be the symmetry epipolar distance \cite{Hartley_MVG_2003}:
\begin{equation}\label{eq:7}
\small
\mathcal{L} ( {\textit{\textbf{F}}} ; \ibu_i, \ibu_i') 
\!=\! {\left| {\textit{\textbf{u}}'_i{^T}} \textit{\textbf{F}} \textit{\textbf{u}}_i \right|}{\left( {\frac{1}{{\left\| ({\textit{\textbf{F}}{^T} \textit{\textbf{u}}'_i})^{\rm{(1:2)}} \right\|_2}} \!+\! \frac{1}{{\left\| ({\textit{\textbf{F}} \textit{\textbf{u}}_i})^{\rm{(1:2)}} \right\|_2}}} \right)}.
\end{equation}
We tested several variants of distance functions including the Sampson distance and algebraic distance, and decided to use the symmetry epipolar distance, because it showed superior results in the experiments. Interestingly, these findings contrast with the findings of Ref. \cite{Hartley_MVG_2003}.

By training through minimizing the loss function, we can train the network without any ground truth data at all, contrary to Ref. \cite{Pais_3DRegNet_arxiv_2019} or Ref. \cite{Ranftl_DeepFundamental_ECCV_2018}; the network achieves self-supervisory in the geometric sense.
It also enables us to exploit a very large number of frames from video sequence datasets under various kinds of camera motion.

\vspace{0.5mm}
\noindent\textbf{Addressing the ordering invariance.}
Our network model is designed to be invariant to the order of the input image points similar to Ref. \cite{Qi_PointNet_CVPR_2017} or Ref. \cite{Pais_3DRegNet_arxiv_2019}, thereby obtaining invariance in the subsequent fundamental matrix computation.  
\begin{proposition}\label{p2}
 As long as the computation of the normalization matrices $\textit{\textbf{T}}'$ and $\textit{\textbf{T}}$ has permutation invariance, then so has the computation of the fundamental matrix. 
\end{proposition}

\begin{proof}
	Because $\textit{\textbf{T}}'$ and $\textit{\textbf{T}}$ maintain invariant for any order of the input data $\textit{\textbf{u}}'$ and $\textit{\textbf{u}}$, the resulting $\hat{\textit{\textbf{u}}}'$ and $\hat{\textit{\textbf{u}}}$ hold the same order as $\textit{\textbf{u}}'$ and $\textit{\textbf{u}}$ after normalization; this is equivalent to performing a row transformation on the transformed coefficient matrix $\hat{{\textit{\textbf{A}}}}$ in Eq.~\eqref{eq:4} for different orders of $\textit{\textbf{u}}'$ and $\textit{\textbf{u}}$.
	However, when the row transformation is made to $\hat{{\textit{\textbf{A}}}}$, the right singular vector corresponding to the smallest singular value of $\hat{{\textit{\textbf{A}}}}$ does not change \cite{Horn_MatrixAnalysis_2012}, {i.e.} the estimation of $\hat{{\textit{\textbf{F}}}}$ is not affected. Furthermore, the final fundamental matrix ${{\textit{\textbf{F}}}}$ also has permutation invariance.
\end{proof}



\vspace{0.5mm}
\noindent\textbf{Training procedure.}
The network is implemented in PyTorch. We adopt the Adamax Optimizer \cite{Kingma_Adam_ICLR_2015} with an initial learning rate of $10^{-3}$ and a decreasing learning rate of 0.8 times per 10 epochs. The chosen batch size is 16 and the network is trained for 150 epochs. Each input set is pre-filtered by the residual based on the original {eight-point algorithm} with a threshold (60 pixels) sufficiently large to enhance the stability of the training process. 

\section{Experimental results} \label{sec:experiments}
To prove that our approach can learn normalization matrices adapted to the input data and obtain more accurate fundamental matrix estimations, we benchmark the performance of our approach on three typical datasets with varying regularity. Furthermore, we perform cross-dataset validation to prove the generalizability of our approach. 

\subsection{Datasets}

\noindent{\textbf{KITTI dataset.}} The KITTI odometry dataset \cite{Geiger_KITTI_CVPR_2012} consists of 22 distinct sequences from a car driving around a residential area. This dataset exhibits dominant forward motion with high regularity but shows difficult data associations. We choose the first 11 sequences with ground truth from GPS and a Velodyne LiDAR. Specifically, we employ sequences ``{00}'' to ``{05}'' for training and sequences ``{06}'' to ``{10}'' for testing in our experiment, which enables a fair comparison with recent state-of-the-art methods \cite{Ranftl_DeepFundamental_ECCV_2018}.

\noindent{\textbf{TUM dataset.}} We use the indoor sequences from the TUM RGB-D dataset \cite{Sturm_RGBD_IROS_2012}, which contains several hand-held sequences with ground truth obtained by an additional motion capture system. This dataset reflects rich camera motion and scene geometry, and shows the most general cases for fundamental matrix estimation. 
We exploit the cross-validation for the sequence ``{fr3\_long\_office}'' during training. 
To better test the generalizability of the proposed method,
we resize the image size of the TUM RGB-D dataset to be consistent with that of the KITTI dataset.

\noindent{\textbf{Cambridge dataset.}} The Cambridge dataset \cite{Kendall_PoseNet_ICCV_2015} is a large-scale outdoor urban localization setting, containing six challenging scenes with changes in perspective and illumination; this setting is quite different from TUM and KITTI datasets. Here we adopt the ``{St Mary's Church}'' scene to evaluate the generalization ability of our proposed approach, and report only the qualitative results in the following section.

We generate two different correspondence datasets for each of the KITTI dataset and the TUM dataset, which are stored in a manner similar to that used in Ref. \cite{Menze_SceneFlow_CVPR_2015}.
First, 1000 correspondences based on SIFT \cite{Lowe_SIFT_IJCV_2004} are pre-filtered by employing a ratio test with a threshold of 0.8.
The second one does not leverage the ratio test to pre-filter the correspondences, which generates a challenging dataset with high noise.
The ratio test is a frequently used strategy for improving the robustness and accuracy of feature matching. 
Therefore, unless otherwise stated, we utilize pre-filtered datasets in our experiments.
Moreover, each input sample is generated by shuffling all the correspondences between two views in the dataset.

\subsection{Evaluation protocols}
To evaluate the performance of our approach, we report the average better rate of per input sample, {i.e.}, the average percentage that our learning-based normalization outperforms Hartley's normalization in terms of the symmetric epipolar distance (see Eq.~\eqref{eq:7}).
Besides, in the experiments within the RANSAC framework, we evaluate the average percentage of inliers (correspondences with errors less than 1 pixel or 0.1 pixels),
as well as the \textit{F}1 (the average percentage of correspondences below 1 pixel error with respect to the ground truth epipolar line).



\begin{figure*}[!t]
	\begin{minipage}[b]{.5\linewidth}
		\centering
		\includegraphics[width=0.913\textwidth]{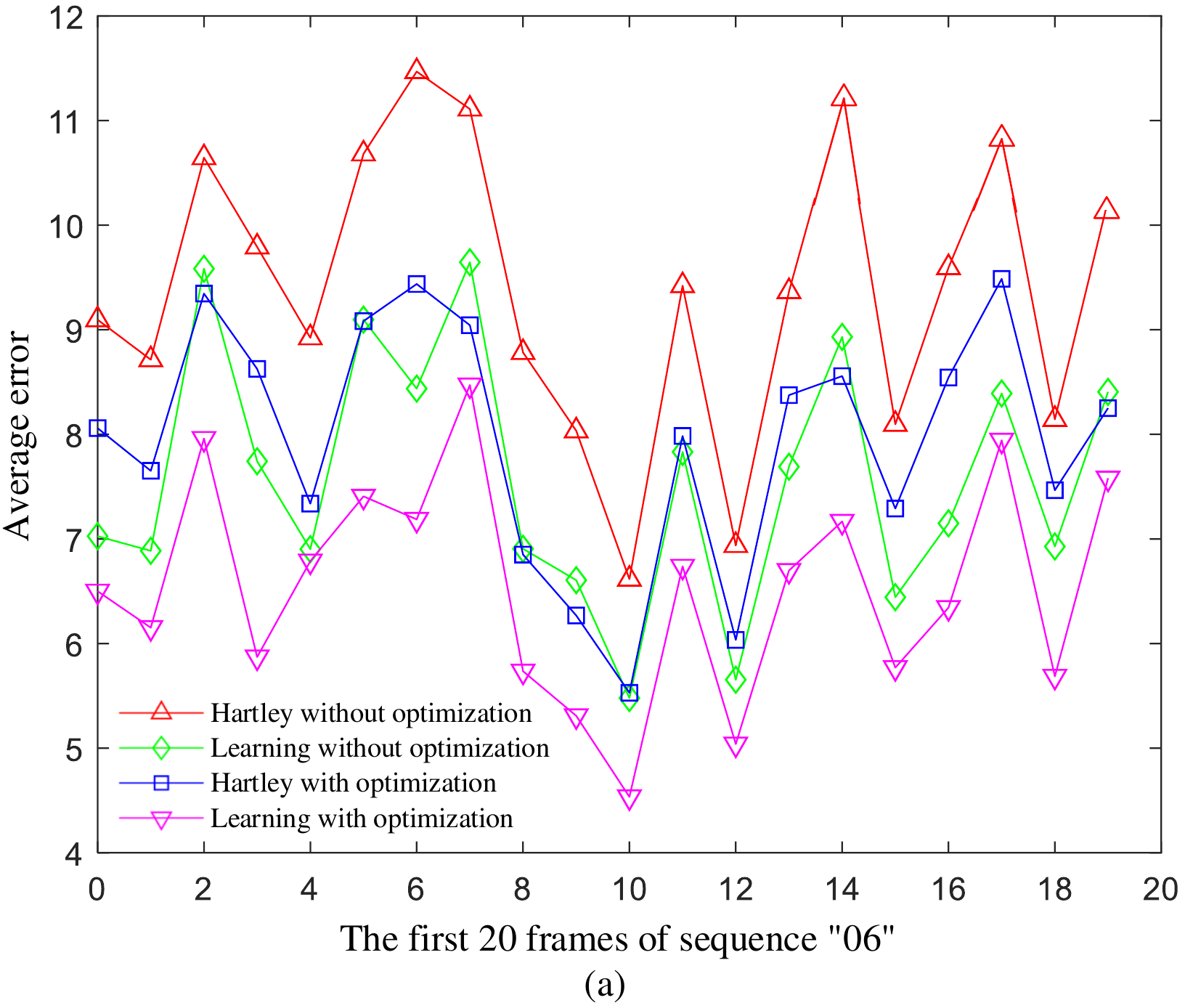}
	\end{minipage}
	\begin{minipage}[b]{.5\linewidth}
		\centering
		\includegraphics[width=0.9389\textwidth]{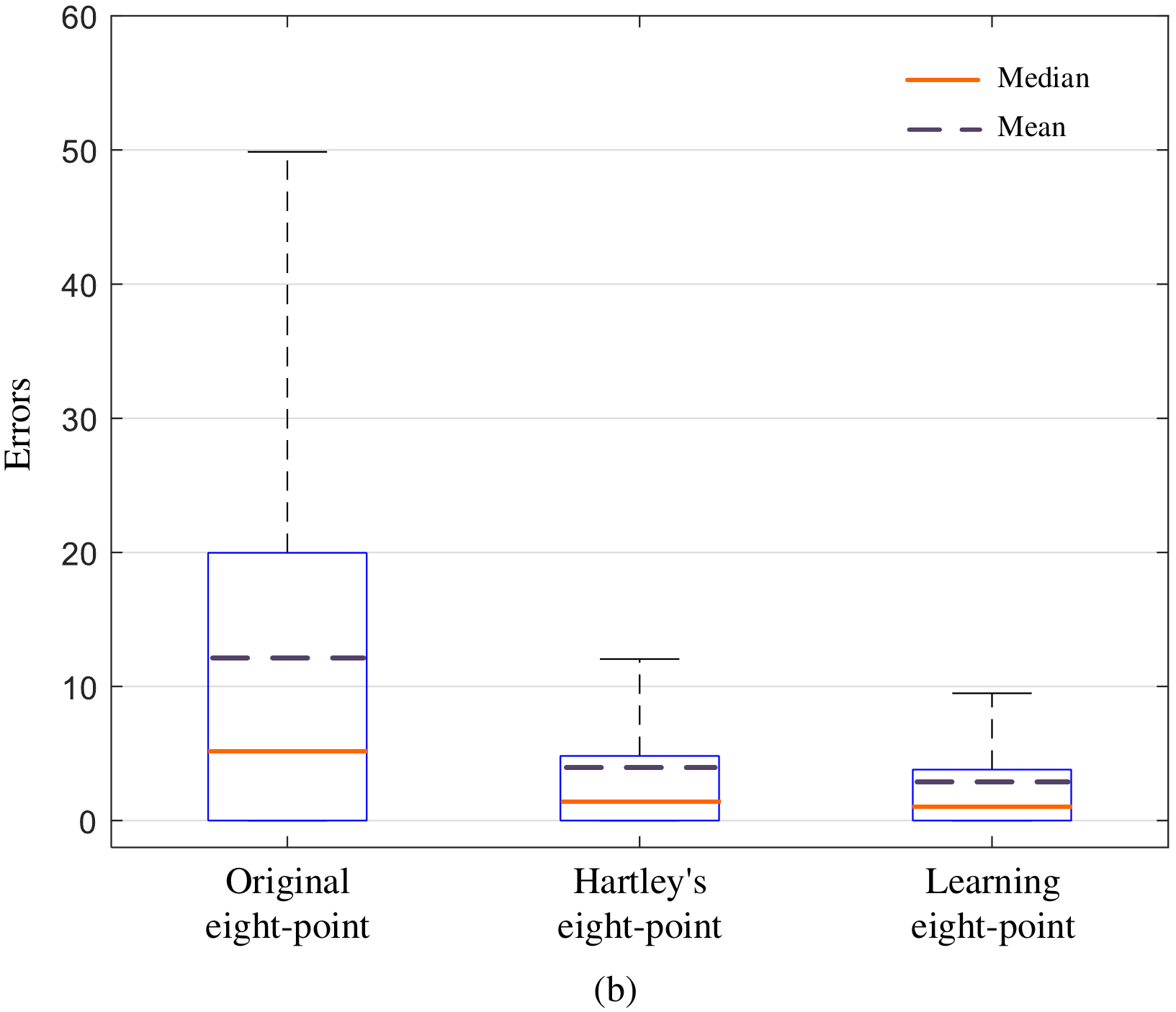}
	\end{minipage}%
	\caption{(a) Average pixel errors of per sample with or without optimization of the first 20 frames of sequence ``{06}''. Our direct results are almost the same as those based on Hartley with optimization. (b) Average pixel error of each sample for the different eight-point methods. We discard the input samples with an original eight-point error greater than 60 for better visualization. \label{fig:fig4}}
\end{figure*}

\begin{figure*}[!t]
	\centering
	\includegraphics[width=0.723\textwidth]{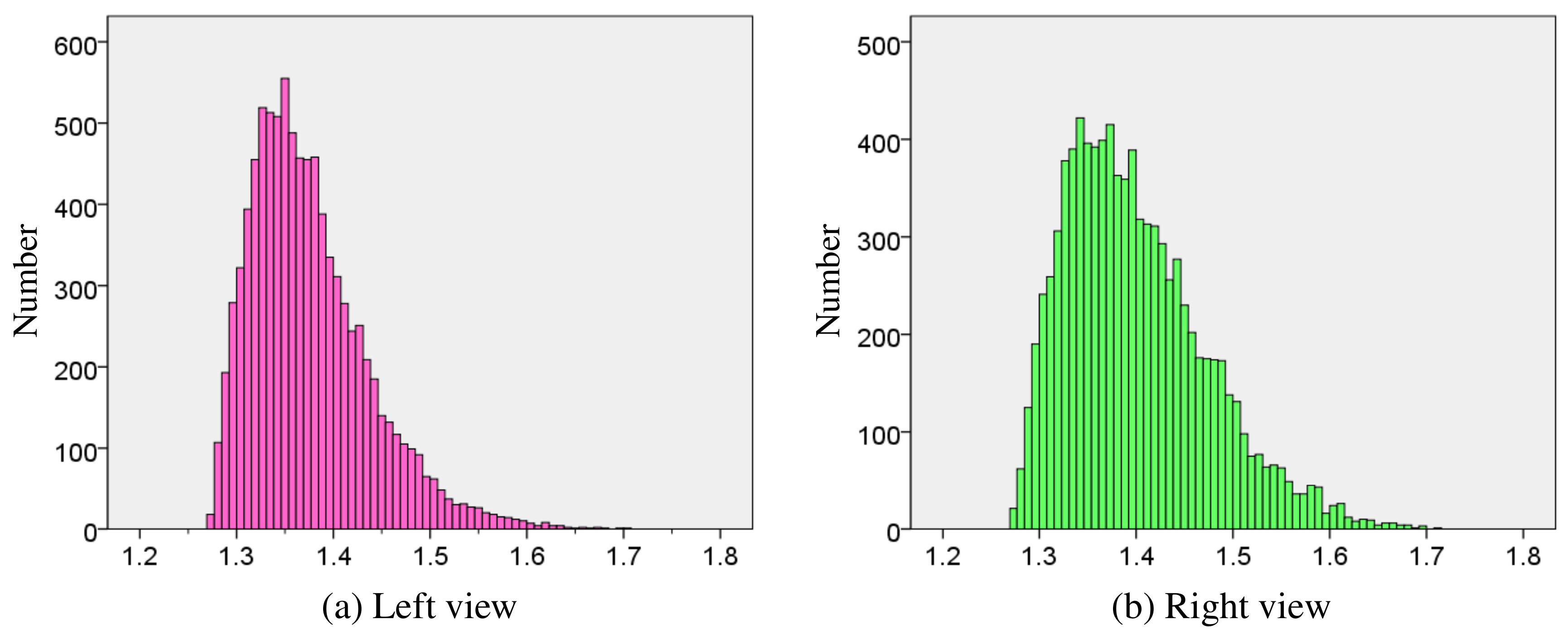}
	\caption{Learning-based normalized distances from the origin on the left and right camera views, respectively. Hartley's normalization makes them always $\sqrt{2}$, while our approach learns a robust normalization scheme adapted to the input data.\label{fig:fig5}}
    \vspace{-2mm}
\end{figure*}

\subsection{Experimental evaluations}
In the first experiment, we evaluate the performance of our approach on per input sample. 
We first optimize ${\textit{\textbf{F}}}$ based on Eq.~\eqref{eq:7} under singularity constraints for Hartley's normalization and our learning-based normalization in the KITTI testing set, and the results are summarized in Fig.~\ref{fig:fig4} (a). The equivalence between our approach and Hartley-based optimization result is reported, which indicates that our approach can provide better initial values for more sophisticated nonlinear optimization methods.
Unlike the constant $\sqrt{2}$ distance from the origin in Hartley's normalization, Fig.~\ref{fig:fig5} shows that our learning-based normalization predicts a distance tailored to each input data, which exploits the inherent regularity of the input data.

\begin{table}[!t]
	\caption{Results of the average improvement rate of per input sample in diverse training sets. 
		Our approach not only takes into account the inherent regularity of the input data but also learns a better and more generalized normalization scheme.}\label{t1}
	\centering
	\resizebox{1.0\linewidth}{!}{
		\begin{tabular}{lcc}
			\hline
			& {@KITTI} & {@TUM} \\
			& Better rate (\%)       & Better rate (\%)  \\ \cline{2-3}
			Train on KITTI & 86.68            & 90.04       \\
			Train on TUM   & 84.22            & 89.73       \\ 
			Train on KITTI \& TUM & 87.05            & 91.43       \\ \hline
	\end{tabular}}
\end{table}

Then, we quantitatively evaluate the average improvement rate of per each input sample, which is our primary concern.
Since Hartley's normalization is the most widely-used normalization method \cite{Hartley_MVG_2003}, we only compare with it here.
As presented in Table~\ref{t1}, our learning-based normalization outperforms Hartley's normalization for each input sample. Interestingly, the model trained on the KITTI dataset is generalizable well to the TUM dataset, and vice versa, which shows the great generalizability of our approach.
To further analyze the impact of training sets on our approach, we provide experimental results by evaluating the average percentage of each input sample when using KITTI and TUM datasets jointly as training sets. The performance of our approach is further improved for each input sample, which shows that our approach can learn a better and more generalized normalization scheme from more training data that contains diverse regularities.
Finally, in Fig.~\ref{fig:fig4} (b), we report the distribution of the symmetric epipolar distance for the original eight-point algorithm, with Hartley's normalization, and with our learning-based normalization. While both have achieved great improvements with respect to the un-normalization version, our learning-based normalization consistently outperforms Hartley's normalization in achieving lower errors for eight input correspondences.

\begin{table}[!t]
	\centering
	\caption{Results on the KITTI testing set with the ratio test under different inlier thresholds. }\label{t2}
	\resizebox{0.90\linewidth}{!}{
		\begin{tabular}{lcccc}
			\hline
			\textit{} & \multicolumn{2}{c}{{@0.1px}} & \multicolumn{2}{c}{{@1px}}        \\ \cline{2-5}
			& Inliers (\%)        & \textit{F}1    & Inliers (\%)   & \textit{F}1   \\ \hline
			Ranftl's \cite{Ranftl_DeepFundamental_ECCV_2018}    & 24.61          & 14.65      & 85.87         & 75.77       \\
			MLESAC \cite{Torr_MLESAC_CVIU_2000}    & 18.60          & 12.54      & 84.48         & 75.15        \\
			LMEDS \cite{Rousseeuw_LMSR_JASA_1984}     & 20.01          & 13.34      & 84.23         & 75.44        \\
			USAC \cite{Raguram_USAC_PAMI_2013}      & 21.43          & 13.90      & 85.13         & 75.70        \\
			RANSAC \cite{Fischler_RANSAC_ACM_1981}    & 21.85          & 13.84      & 84.96         & 75.65        \\
			Ours      & 21.89          & 13.86      & 84.98         & 75.66        \\
			\hline
	\end{tabular}}
\end{table}

From the superior performance of our learning-based normalization algorithm over each input sample, we further heuristically verify that our approach can be effectively integrated into the traditional RANSAC framework \cite{Fischler_RANSAC_ACM_1981}.
%
In the experimental comparison, we follow the most related and classic work \cite{Ranftl_DeepFundamental_ECCV_2018}. We compare our approach with the least median of squares (LMEDS) \cite{Rousseeuw_LMSR_JASA_1984}, MLESAC \cite{Torr_MLESAC_CVIU_2000}, USAC \cite{Raguram_USAC_PAMI_2013}, Ranftl's method \cite{Ranftl_DeepFundamental_ECCV_2018} and RANSAC \cite{Fischler_RANSAC_ACM_1981}, where RANSAC is based on Hartley's normalization while our approach is performed with the learning-based normalization. 
Note that USAC is a state-of-the-art robust estimation framework, and ``RANSAC + normalized eight-point algorithm'' represents the gold standard \cite{Hartley_MVG_2003} for geometric tasks such as visual odometry and SLAM.
Inside Ranftl's method \cite{Ranftl_DeepFundamental_ECCV_2018}, the matching scores have been used as additional information to guide the estimation, which can result in an obvious improvement in average accuracy. 
By contrast, we leverage only the original RANSAC to conduct experiments for performance evaluation.
It is also worth noting that as a supervised learning-based framework, Ranftl's method requires ground truth correspondences in training, while our approach is fully self-supervised.
Additionally, designing an ensemble network to improve overall performance such as DSAC \cite{Brachmann_DSAC_CVPR_2017} is outside the scope of this paper, as our focus is better normalization for each sample. 

Table~\ref{t2} summarizes the results on the KITTI dataset. 
Within the RANSAC framework, our learning-based normalization performs on par with Hartley's normalization on the KITTI benchmark.
Furthermore, we evaluate the performance based on the challenging testing set without the ratio test, and the results are presented in Table~\ref{t3}.
Note that our approach achieves higher inliers on the TUM dataset.
We remark here that, recent analyses in Refs. \cite{chin2018robust,chin2020quantum} as well as related experiments in Ref. \cite{ding2020minimal} indicate that the RANSAC paradigms with supporting heuristics can only increase the chance of finding the final good solution and are not completely governed by the internal solver,
which is one possible reason for the slight improvement of our method when it is embedded into RANSAC.
%
%
Overall, the effectiveness of our learning-based normalization method combined with RANSAC is demonstrated.



\begin{table}[!t]
	\centering
	\caption{Performance of the proposed method combined with RANSAC in the testing set without the ratio test.}\label{t3}
	\resizebox{0.75\linewidth}{!}{
		\begin{tabular}{cccc}
			\hline
			\textit{}              &                      & {@0.1px}                 & {@1px}                          \\
			\multicolumn{1}{l}{}   & \multicolumn{1}{l}{} & \multicolumn{1}{l}{Inliers (\%)} & \multicolumn{1}{l}{Inliers (\%)} \\ \hline
			\multirow{2}{*}{KITTI} & Hartley              & 4.40                           & 30.50                          \\
			& Learning             & 4.38                           & 30.55                          \\ \cline{2-4} 
			\multirow{2}{*}{TUM}   & Hartley              & 7.60                           & 44.13                          \\
			& Learning             & 7.64                           & 44.18                          \\ \hline
	\end{tabular}}
\end{table}

Finally, we directly employ the network model trained on the KITTI dataset, which is very different from the Cambridge dataset. The qualitative generalization results for the Cambridge dataset are reported in Fig.~\ref{fig:fig10}.
One can see that our approach can achieve an accurate two-view fundamental matrix estimation, which reflects the good generalization ability of our approach.
Moreover, since we always centralize the correspondences first, varying image sizes and distributions of features will not have a significant impact on the final results.
Currently, our forward propagation time is approximately 5 times that of Hartley's normalization due to the use of 12-layer ResNet architectures.
Fortunately, these efficiency sacrifices can improve normalization to achieve more accurate epipolar geometry for each sample.

\begin{figure*}[!t]
	\centering
	\includegraphics[width=0.9765\textwidth]{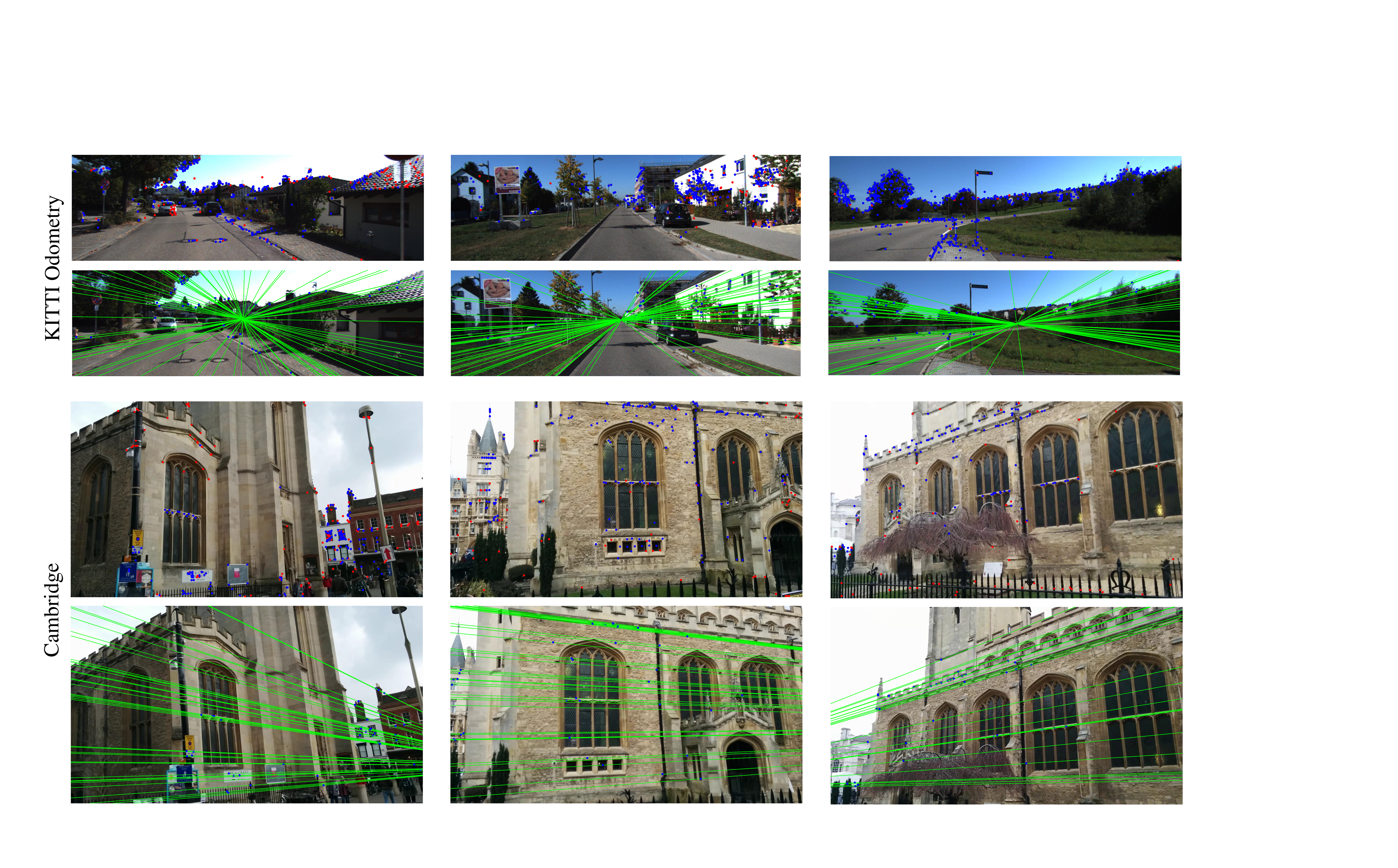}\vspace{-1.5mm}
	\caption{Image pairs from the KITTI and Cambridge datasets. Odd row: First image with inliers (blue) and outliers (red). Even row: The estimated epipolar lines of a random subset of inliers in the second image. The images are scaled for visualization.\label{fig:fig10}}
 \vspace{-3.7mm}
\end{figure*}

\vspace{0.5mm}
\noindent\textbf{Influence of the number of correspondences.} We perform additional experiments to analyze the influence of the number of correspondences in the input. We take the median of 1000 trials based on a random testing image. The results are shown in Fig.~\ref{fig:fig11_eccv}, which indicate that better fundamental matrices can be obtained with an increasing number of correspondences. 

\vspace{0.5mm}
\noindent\textbf{Condition numbers.} We conduct another experiment to compare the condition numbers in solving the fundamental matrix and the results are reported in Fig.~\ref{fig:fig12_eccv}. We observe that better numerical conditioning of the transformed coefficient matrix can be obtained by our learning-based normalization, which is one of the keys to our upgraded performance.

\vspace{0.5mm}
\noindent\textbf{Nonlinear projection.} The singularity of $\hat{\ibF}$ is evaluated by calculating $\rho=r_2/r_3$ for every 100 consecutive frames of the KITTI testing set. The results are displayed in Fig.~\ref{fig:fig13_eccv}, which shows our learning-based approach is able to achieve smaller nonlinear projection errors.
These findings also verify our argument that the condition number of the transformed coefficient matrix via a better normalization will be more conducive to imposing the singularity constraint on the resulting fundamental matrix.
Note that these experimental results all highlight the superiority of our learning-based normalization approach.

\begin{figure}[htbp]
	\centering
	\includegraphics[width=0.391\textwidth]{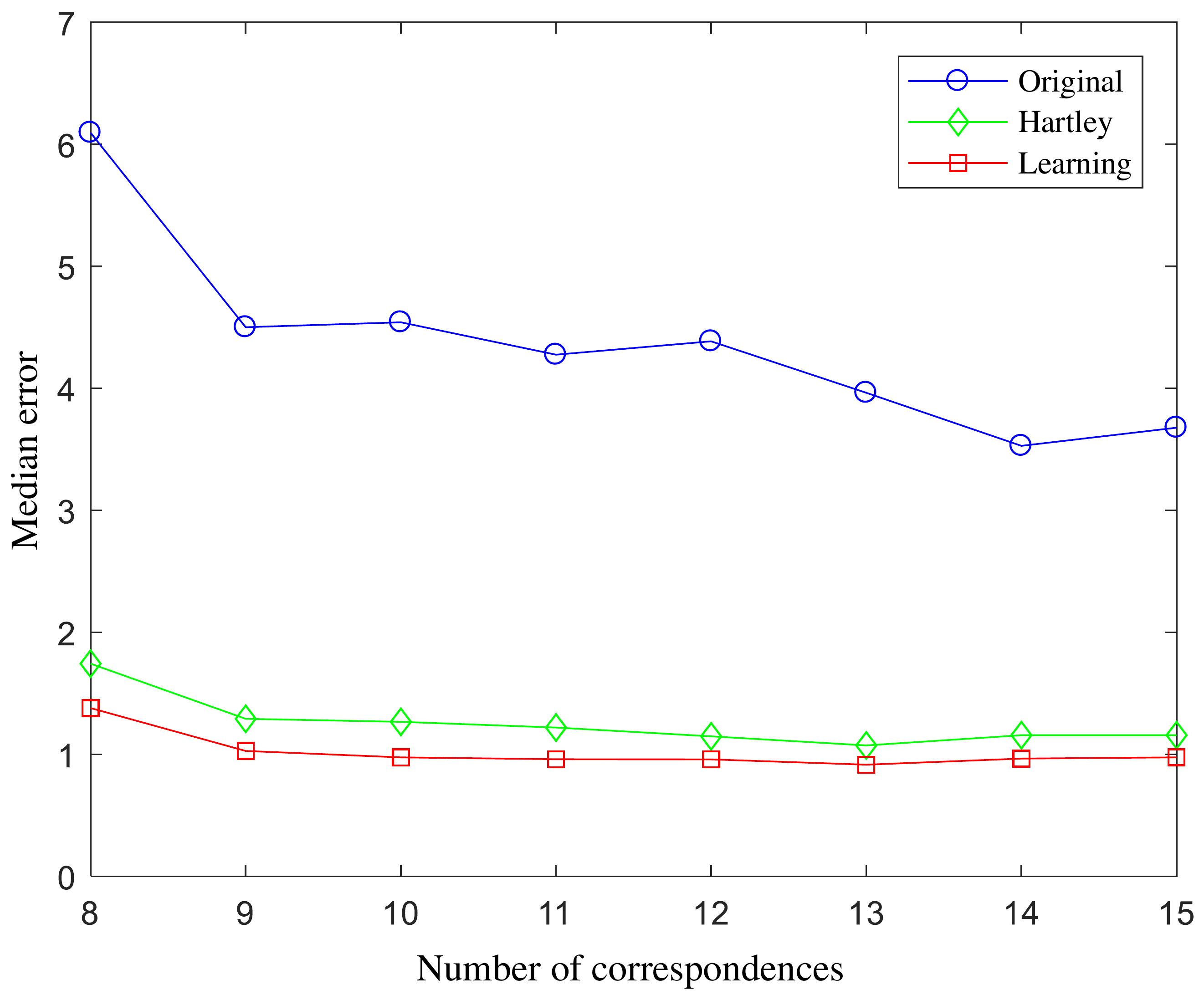}
	\caption{Influence of the number of correspondences on the median error of 1000 trials in a random testing frame. \label{fig:fig11_eccv}}
 \vspace{-2.5mm}
\end{figure}
\begin{figure}[htbp]
	\centering
	\includegraphics[width=0.414\textwidth]{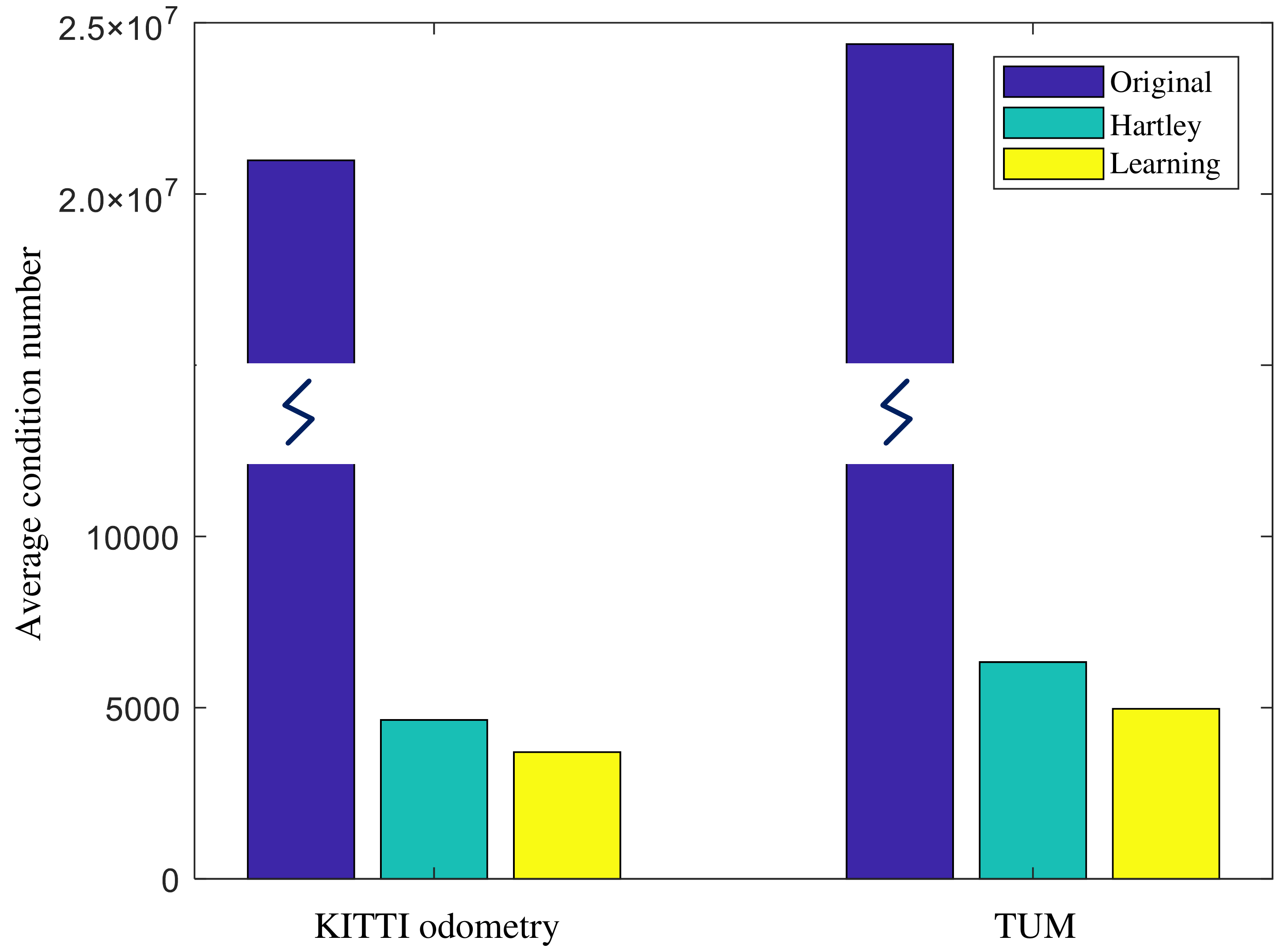}
	\caption{Effect of diverse normalization schemes on the average condition number. \label{fig:fig12_eccv}}
 \vspace{-2.5mm}
\end{figure}
\begin{figure}[htbp]
	\centering
	\includegraphics[width=0.412\textwidth]{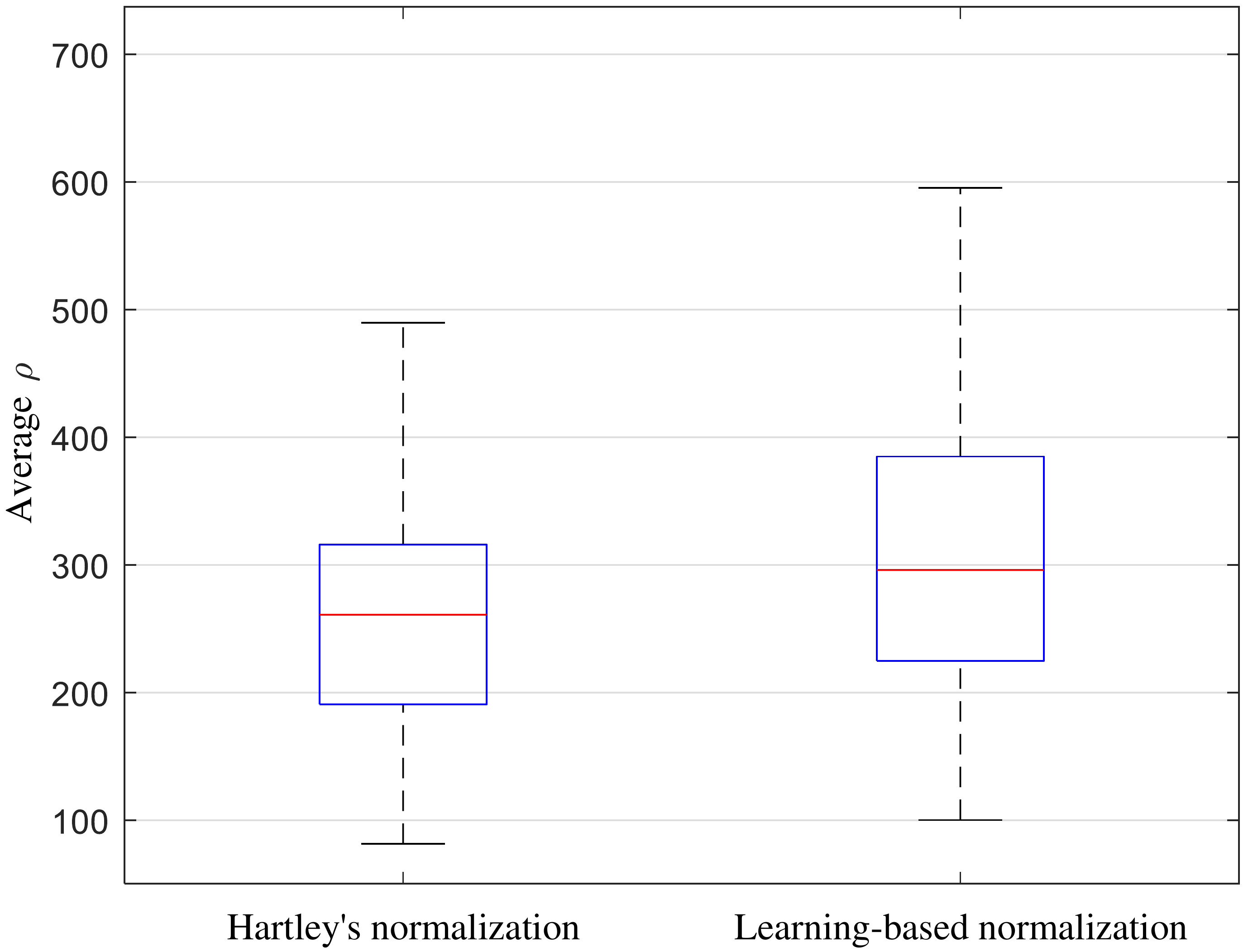}
	\caption{Average $\rho$ of every 100 consecutive samples in the KITTI testing set.\label{fig:fig13_eccv}}
 \vspace{-2.5mm}
\end{figure}



\section{Conclusion}
In this paper, we revisit the classic two-view geometry computation with eight point correspondences and employ CNNs to provide a novel perspective for better normalization. First, we present that the ideal condition number can be obtained by our approach to be more consistent with the following singularity constraint enforcement step. Second, we propose a self-supervised deep neural network to learn a robust normalization scheme for more accurate fundamental matrix estimation. Our approach enables a data-driven estimation pipeline to perform interpretable and generalized fundamental matrix estimation. Our learning-based normalization solution is superior to Hartley's normalization for each input sample,
and is comparable to Hartley's normalization when integrated with RANSAC. 
Its potential advantage is to provide better initial values for non-linear optimization and to afford better interpretability for an ensemble network.
In the future, we plan to design a lightweight network to weigh time and quality, utilize ground truth correspondences or ground truth matching scores to explore supervised two-view geometry estimation, and further extend our deep solution to other multi-view geometry problems such as triangulation, trifocal tensor estimation, etc.


\begin{small}
\vspace{.3in} \noindent \textbf{Abbreviations}

\noindent{2D, two-dimensional; 3D, three-dimensional; CNN, convolutional neural network; DLT, direct linear transformation; RANSAC, random sample consensus; SCE, singularity constraint enforcement; SfM, structure from motion; SLAM, simultaneous localization and mapping; SVD, singular value decomposition.
}

\vspace{.3in} \noindent \textbf{Acknowledgements}

\noindent{We want to thank Jihuang Dai and Xiang Guo for investigating relevant literature. The authors express their gratitude to the anonymous reviewers and the editor.
}

\vspace{.3in} \noindent \textbf{Availability of data and materials}

\noindent{The datasets generated during and/or analyzed during the current study 
are available in the KITTI repository: 
\url{https://www.cvlibs.net/datasets/kitti/eval_odometry.php}.
}

\vspace{.3in} \noindent \textbf{Funding}

\noindent{This work was supported in part by the National Natural Science Foundation of China (No. 62271410) and the National Postdoctoral Innovative Talent Program, China (No. BX20230013).
}


\section*{Declarations}

\vspace{.3in} \noindent \textbf{Author contributions}

\noindent{All authors contributed to the study conception and design. Material preparation, theoretical derivation and analysis were performed by FB and DYC. The first draft of the manuscript was written by FB and all authors commented on previous versions of the manuscript. All authors read and approved the final manuscript.
}

\vspace{.3in} \noindent \textbf{Author details}

\noindent{$^1$School of Electronics and Information, Northwestern Polytechnical University and Shaanxi Key Laboratory of Information Acquisition and Processing, Xi'an 710129, China. 
$^2$Department of Art and Technology, Sogang University, Seoul 04107, Korea. 
}

\vspace{.3in} \noindent \textbf{Competing interests}

\noindent{All authors certify that they have no affiliations with or involvement in any organization or entity with any financial or non-financial interest in the subject matter or materials discussed in this manuscript.
}

\end{small}

\bibliographystyle{unsrt}


\begin{thebibliography}{1}


\bibitem{Longuet_Reconstructing_Nature_1981}
Longuet-Higgins, H. C. (1981).
\newblock {A computer algorithm for reconstructing a scene from two projections}.
\newblock {\em Nature}, {\em 293}(5828), 133--135.

\bibitem{Hartley_Normalization_TPAMI_1997}
Hartley, R. (1997).
\newblock {In defense of the eight-point algorithm}.
\newblock {\em IEEE Transactions on Pattern Analysis and Machine Intelligence}, {\em 19}(6), 580–593.

\bibitem{dai2016rolling}
Dai, Y., Li, H., \& Kneip, L. (2016).
\newblock {Rolling shutter camera relative pose: Generalized epipolar geometry}.
\newblock In {\em Proceedings of the IEEE conference on computer vision and pattern recognition} (pp. 4132--4140). Piscataway: IEEE.

\bibitem{zhao2021homography}
Zhao, C., Fan, B., Hu, J., Pan, Q., \& Xu, Z. (2021).
\newblock {Homography-based camera pose estimation with known gravity direction for UAV navigation}.
\newblock {\em Science China Information Sciences}, {\em 64}(1), 1--13.

\bibitem{szpak2015guaranteed}
Szpak, Z. L., Chojnacki, W., \& van den Hengel, A. (2015).
\newblock {Guaranteed ellipse fitting with a confidence region and an uncertainty measure for centre, axes, and orientation}.
\newblock {\em Journal of Mathematical Imaging and Vision}, {\em 52}(2), 173--199.

\bibitem{zhang2014structure}
Zhang, L., \& Koch, R. (2014).
\newblock {Structure and motion from line correspondences: Representation, projection, initialization and sparse bundle adjustment}.
\newblock {\em Journal of Visual Communication and Image Representation}, {\em 25}(5), 904--915.

\bibitem{Muhlich_Subspace_SCIA_2001}
M{\"u}hlich, M., \& Mester, R. (2001).
\newblock {Subspace methods and equilibration in computer vision}.
\newblock In {\em Proceedings of the 12th Scandinavian conference on image analysis} (pp. 415--422). Cham: Springer.

\bibitem{Elamr_ErrorPropagation_IROS_2013}
Mair, E., Suppa, M., \& Burschka, D. (2013).
\newblock {Error propagation in monocular navigation for {Z}$_\infty$ compared to eightpoint algorithm}.
\newblock In {\em Proceedings of the IEEE/RSJ international conference on intelligent robots and systems} (pp. 4220--4227). Piscataway: IEEE.

\bibitem{daSilveira_Perturbation_CVPR_2019}
da Silveira, T. L., \& Jung, C. R. (2019).
\newblock {Perturbation analysis of the 8-Point algorithm: A case study for wide {FoV} cameras}.
\newblock In {\em Proceedings of the IEEE/CVF conference on computer vision and pattern recognition} (pp. 11757--11766). Piscataway: IEEE.

\bibitem{Muhlich_TLS_ECCV_1998}
M{\"u}hlich, M., \& Mester, R. (1998).
\newblock {The role of total least squares in motion analysis}.
\newblock In {\em Proceedings of the 5th European conference on computer vision} (pp. 305--321). Cham: Springer.

\bibitem{Detone_DeepHomography_arxiv_2016}
DeTone, D., Malisiewicz, T., \& Rabinovich, A. (2016).
\newblock {Deep image homography estimation}.
\newblock arXiv preprint. arXiv:1606.03798.

\bibitem{Ranftl_DeepFundamental_ECCV_2018}
Ranftl, R., \& Koltun, V. (2018).
\newblock {Deep fundamental matrix estimation}.
\newblock In {\em Proceedings of the 15th European conference on computer vision} (pp. 284--299). Cham: Springer.

\bibitem{Tang_BANet_ICLR_2018}
Tang, C., \& Tan, P. (2018).
\newblock {{BA-Net}: Dense bundle adjustment network}.
\newblock In {\em Proceedings of the 6th international conference on learning representations} (pp. 284--299). Retrieved October 7, 2023, from https://openreview.net/forum?id=B1gabhRcYX.

\bibitem{Sunghoon_DPSNet_ICLR_2019}
Sunghoon I., Hae-Gon J., Stephen L., \& In S. K. (2019).
\newblock {{DPSNet}: End-to-end deep plane sweep stereo}.
\newblock [Poster presentation]. {\em Proceedings of the 7th international conference on learning representations}. New Orleans, USA.

\bibitem{fan2021rs}
Fan, B., Wang, K., Dai, Y., \& He, M. (2021).
\newblock {RS-DPSNet: Deep plane sweep network for rolling shutter stereo images}.
\newblock {\em IEEE Signal Processing Letters}, {\em 28}, 1550--1554.

\bibitem{fan2023rolling}
Fan, B., Dai, Y., \& He, M. (2023).
\newblock {Rolling shutter camera: Modeling, optimization and learning}.
\newblock {\em Machine Intelligence Research}, {\em 20}(6), 783--798.

\bibitem{fan2022rolling}
Fan, B., Dai, Y., \& Li, H. (2022).
\newblock {Rolling shutter inversion: Bring rolling shutter images to high framerate global shutter video}.
\newblock {\em IEEE Transactions on Pattern Analysis and Machine Intelligence}, {\em 45}(5), 6214--6230.

\bibitem{Brachmann_DSAC_CVPR_2017}
Brachmann, E., Krull, A., Nowozin, S., Shotton, J., Michel, F., \& Gumhold, S. (2017).
\newblock {{DSAC}-differentiable {RANSAC} for camera localization}.
\newblock In {\em Proceedings of the IEEE conference on computer vision and pattern recognition} (pp. 6684--6692). Piscataway: IEEE.

\bibitem{Csurka_Uncertainty_CVIU_1997}
Csurka, G., Zeller, C., Zhang, Z., \& Faugeras, O. D. (1997).
\newblock {Characterizing the uncertainty of the fundamental matrix}.
\newblock {\em Computer Vision and Image Understanding}, {\em 68}(1), 18--36.

\bibitem{Frederic_Uncertainty_BMVC_2008}
Sur, F., Noury, N., \& Berger, M.-O. (2008).
\newblock {Computing the uncertainty of the 8 point algorithm for fundamental matrix estimation}.
\newblock In {\em Proceedings of the British machine vision conference} (pp. 965--974). Swansea: BMVA Press.

\bibitem{Chojnacki_Revisiting8pt_TPAMI_2003}
Chojnacki, W., \& Brooks, M. J. (2003).
\newblock {Revisiting {Hartley's} normalized eight-point algorithm}.
\newblock {\em IEEE Transactions on Pattern Analysis and Machine Intelligence}, {\em 25}(9), 1172--1177.

\bibitem{Chojnacki_Consistency_JMIV_2007}
Chojnacki, W., \& Brooks, M. J. (2007).
\newblock {On the consistency of the normalized eight-point algorithm}.
\newblock {\em Journal of Mathematical Imaging and Vision}, {\em 28}(1), 19--27.

\bibitem{Nguyen_DeepHomographyUnsupervised_Robot&Automation_2018}
Nguyen, T., Chen, S. W., Shivakumar, S. S., Taylor, C. J., \& Kumar, V. (2018).
\newblock {Unsupervised deep homography: A fast and robust homography estimation model}.
\newblock {\em IEEE Robotics and Automation Letters}, {\em 3}(3), 2346--2353.

\bibitem{Omid_DeepFundamental_wo_corresponences_ECCV_2018}
Poursaeed, O., Yang, G., Prakash, A., Fang, Q., Jiang, H., \& Hariharan, B. (2018).
\newblock {Deep fundamental matrix estimation without correspondences}.
\newblock In {\em Proceedings of the 15th European conference on computer vision} (pp. 485--497). Cham: Springer.

\bibitem{Yi_LearningCorrespondences_CVPR_2018}
Yi, K. M., Trulls, E., Ono, Y. Lepetit, V., Salzmann, M., \& Fua, P. (2018).
\newblock {Learning to find good correspondences}.
\newblock In {\em Proceedings of the IEEE/CVF conference on computer vision and pattern recognition} (pp. 2666--2674). Piscataway: IEEE.

\bibitem{Probst_2019_CVPR}
Probst, T., Paudel, D. P., Chhatkuli, A., \& Gool, L. V. (2019).
\newblock {Unsupervised learning of consensus maximization for {3D} vision problems}.
\newblock In {\em Proceedings of the IEEE/CVF conference on computer vision and pattern recognition} (pp. 929--938). Piscataway: IEEE.

\bibitem{zhang2022learning}
Zhang, Z., Dai, Y., Fan, B., Sun, J., \& He, M. (2022).
\newblock {Learning a task-specific descriptor for robust matching of 3D point clouds}.
\newblock {\em IEEE Transactions on Circuits and Systems for Video Technology}, {\em 32}(12), 8462--8475.

\bibitem{zhang2022searching}
Zhang, Z., Sun, J., Dai, Y., Fan, B., \& Liu, Q. (2022).
\newblock {Searching dense point correspondences via permutation matrix learning}.
\newblock {\em IEEE Signal Processing Letters}, {\em 29}, 1192--1196.

\bibitem{Horn_MatrixAnalysis_2012}
Horn, R. A., \& Johnson, C. R. (2012).
\newblock {Matrix analysis}.
\newblock Cambridge: Cambridge University Press.

\bibitem{Chen_Cond_ECNU_1986}
Chen, D. (1986).
\newblock {Some conclusions on condition numbers of matrix}.
\newblock {\em Journal of East China Normal University (Natural Science)}, {\em 3}(2), 11--18.

\bibitem{He_ResNet_CVPR_2016}
He, K., Zhang, X., Ren, S., \& Sun, J. (2016).
\newblock {Deep residual learning for image recognition}.
\newblock In {\em Proceedings of the IEEE conference on computer vision and pattern recognition} (pp. 770--778). Piscataway: IEEE.

\bibitem{Ulyanov_TextureNetwork_CVPR_2017}
Ulyanov, D., Vedaldi, A., \& Lempitsky, V. (2017).
\newblock {Improved texture networks: Maximizing quality and diversity in feed-forward stylization and texture synthesis}.
\newblock In {\em Proceedings of the IEEE conference on computer vision and pattern recognition} (pp. 6924--6932). Piscataway: IEEE.

\bibitem{Pais_3DRegNet_arxiv_2019}
Pais, G. D., Ramalingam, S., Govindu, V. M., Nascimento, J. C., Chellappa, R., \& Miraldo, P. (2020).
\newblock {{3DRegNet}: A deep neural network for 3D point registration}.
\newblock In {\em Proceedings of the IEEE/CVF conference on computer vision and pattern recognition} (pp. 7193--7203). Piscataway: IEEE.

\bibitem{Hartley_MVG_2003}
Hartley, R., \& Zisserman, A. (2003).
\newblock {Multiple view geometry in computer vision}.
\newblock Cambridge: Cambridge University Press.

\bibitem{Qi_PointNet_CVPR_2017}
Qi, C. R., Su, H., Mo, K., \& Guibas, L. J. (2017).
\newblock {{PointNet}: deep learning on point sets for {3D} classification and segmentation}.
\newblock In {\em Proceedings of the IEEE conference on computer vision and pattern recognition} (pp. 652--660). Piscataway: IEEE.

\bibitem{Kingma_Adam_ICLR_2015}
Kingma, D. P., \& Ba, J. (2015).
\newblock {Adam: A method for stochastic optimization}.
\newblock [Poster presentation]. {\em Proceedings of the 3th international conference on learning representations}. San Diego, USA.

\bibitem{Geiger_KITTI_CVPR_2012}
Geiger, A., Lenz, P., \& Urtasun, R. (2012).
\newblock {Are we ready for autonomous driving? The {KITTI} vision benchmark suite}.
\newblock In {\em Proceedings of the IEEE conference on computer vision and pattern recognition} (pp. 3354--3361). Piscataway: IEEE.

\bibitem{Sturm_RGBD_IROS_2012}
Sturm, J., Engelhard, N., Endres, F., Burgard, W., \& Cremers, D. (2012).
\newblock {A benchmark for the evaluation of {RGB-D SLAM} systems}.
\newblock In {\em Proceedings of the IEEE/RSJ international conference on intelligent robots and systems} (pp. 573--580). Piscataway: IEEE.

\bibitem{Kendall_PoseNet_ICCV_2015}
Kendall, A., Grimes, M., \& Cipolla, R. (2015).
\newblock {{PoseNet}: A convolutional network for real-time 6-dof camera relocalization}.
\newblock In {\em Proceedings of the IEEE international conference on computer vision} (pp. 2938--2946). Piscataway: IEEE.

\bibitem{Menze_SceneFlow_CVPR_2015}
Menze, M., \& Geiger, A. (2015).
\newblock {Object scene flow for autonomous vehicles}.
\newblock In {\em Proceedings of the IEEE conference on computer vision and pattern recognition} (pp. 3061--3070). Piscataway: IEEE.

\bibitem{Lowe_SIFT_IJCV_2004}
Lowe, D. G. (2004).
\newblock {Distinctive image features from scale-invariant keypoints}.
\newblock {\em International Journal of Computer Vision}, {\em 60}(2), 91--110.

\bibitem{Fischler_RANSAC_ACM_1981}
Fischler, M. A., \& Bolles, R. C. (1981).
\newblock {Random sample consensus: A paradigm for model fitting with applications to image analysis and automated cartography}.
\newblock {\em Communications of the ACM}, {\em 24}(6), 381--395.

\bibitem{Rousseeuw_LMSR_JASA_1984}
Rousseeuw, P. J. (1984).
\newblock {Least median of squares regression}.
\newblock {\em Journal of the American Statistical Association}, {\em 79}(388), 871--880.

\bibitem{Torr_MLESAC_CVIU_2000}
Torr, P. H., \& Zisserman, A. (2000).
\newblock {{MLESAC}: A new robust estimator with application to estimating image geometry}.
\newblock {\em Computer Vision and Image Understanding}, {\em 78}(1), 138--156.

\bibitem{Raguram_USAC_PAMI_2013}
Raguram, R., Chum, O., Pollefeys, M., Matas, J., \& Frahm, J.-M. (2012).
\newblock {{USAC}: A universal framework for random sample consensus}.
\newblock {\em IEEE Transactions on Pattern Analysis and Machine Intelligence}, {\em 35}(8), 2022--2038.

\bibitem{chin2018robust}
Chin, T.-J., Cai, Z., \& Neumann, F. (2018).
\newblock {Robust fitting in computer vision: easy or hard?}.
\newblock In {\em Proceedings of the 15th European conference on computer vision} (pp. 701--716). Cham: Springer.

\bibitem{chin2020quantum}
Chin, T.-J., Suter, D., Ch'ng, S.-F., et al. (2020).
\newblock {Quantum robust fitting}.
\newblock In {\em Proceedings of the 15th Asian conference on computer vision} (pp. 485--499). Cham: Springer.

\bibitem{ding2020minimal}
Ding, Y., Yang, J., Ponce, J., et al. (2020).
\newblock {Minimal solutions to relative pose estimation from two views sharing a common direction with unknown focal length}.
\newblock In {\em Proceedings of the IEEE/CVF conference on computer vision and pattern recognition} (pp. 7045--7053). Piscataway: IEEE.


\end{thebibliography}

\end{document}